\documentclass[final]{colt2020}

\usepackage{todonotes}
\allowdisplaybreaks

\title[Algorithms for adversarial linear contextual bandits]{Efficient and robust algorithms for 
\\adversarial linear contextual bandits}
\usepackage{times}
\usepackage[utf8]{inputenc} 
\usepackage[T1]{fontenc}    
\usepackage{hyperref}       
\usepackage{url}            
\usepackage{booktabs}       
\usepackage{amsfonts}       
\usepackage{nicefrac}       
\usepackage{microtype}      
\usepackage{xspace}
\usepackage{amsmath}

\usepackage{mathtools}  
\usepackage{amsmath}
\usepackage{amssymb}
\usepackage{tabulary}
\usepackage{booktabs}

\usepackage{algorithm}
\usepackage{algpseudocode}
\usepackage{algpascal}
\usepackage{comment}
\usepackage{selectp}
\usepackage{selectp}

 \newcommand{\X}{\mathcal{X}}
\newcommand{\F}{\mathcal{F}}
\newcommand{\real}{\mathbb{R}}
\newcommand{\Dw}{\mathcal{D}}

\newcommand{\trace}[1]{\mbox{tr}\left(#1\right)}
\newcommand{\II}[1]{\mathbb{I}_{\left\{#1\right\}}}
\newcommand{\PP}[1]{\mathbb{P}\left[#1\right]}

\newcommand{\EE}[1]{\mathbb{E}\left[#1\right]}
\newcommand{\EEb}[1]{\mathbb{E}\bigl[#1\bigr]}
\newcommand{\EEtb}[1]{\mathbb{E}_t\bigl[#1\bigr]}

\newcommand{\PPt}[1]{\mathbb{P}_t\left[#1\right]}
\newcommand{\EEt}[1]{\mathbb{E}_t\left[#1\right]}

\newcommand{\EEcc}[2]{\mathbb{E}\left[\left.#1\right|#2\right]}

\newcommand{\EEcct}[2]{\mathbb{E}_t\left[\left.#1\right|#2\right]}

\def\argmin{\mathop{\mbox{ arg\,min}}}

\newcommand{\ra}{\rightarrow}

\newcommand{\siprod}[2]{\langle#1,#2\rangle}
\newcommand{\iprod}[2]{\left\langle#1,#2\right\rangle}
\newcommand{\biprod}[2]{\bigl\langle#1,#2\bigr\rangle}

\newcommand{\norm}[1]{\left\|#1\right\|}

\newcommand{\twonorm}[1]{\norm{#1}_2}

\newcommand{\opnorm}[1]{\norm{#1}_{\text{op}}}

\newcommand{\ev}[1]{\left\{#1\right\}}
\newcommand{\pa}[1]{\left(#1\right)}
\newcommand{\bpa}[1]{\bigl(#1\bigr)}

\newcommand{\wh}{\widehat}
\newcommand{\wt}{\widetilde}

\newcommand{\loss}{\ell}
\newcommand{\hloss}{\wh{\loss}}

\newcommand{\lambdamin}{\lambda_{\min}}

\newcommand{\hR}{\wh{R}}

\newcommand{\tX}{X_0}
\newcommand{\htheta}{\wh{\theta}}
\newcommand{\hTheta}{\wh{\Theta}}
\newcommand{\ttheta}{\wt{\theta}}

\newcommand{\Sp}{\Sigma^+}
\newcommand{\Spt}{\Sigma^{+2}}

\newcommand{\transpose}{^\mathsf{\scriptscriptstyle T}}

\definecolor{PalePurp}{rgb}{0.66,0.57,0.66}

\newcommand{\expexpexp}{\textsc{Exp3}\xspace}
\newcommand{\linexp}{\textsc{LinExp3}\xspace}
\newcommand{\linexpreal}{\textsc{RealLinExp3}\xspace}
\newcommand{\linexprobust}{\textsc{RobustLinExp3}\xspace}
\newcommand{\linucb}{\textsc{LinUCB}\xspace}

\coltauthor{%
 \Name{Gergely Neu} \Email{gergely.neu@gmail.com}\\
 \addr Universitat Pompeu Fabra, Barcelona, Spain
 \AND
 \Name{Julia Olkhovskaya} \Email{julia.olkhovskaya@gmail.com}\\
 \addr Universitat Pompeu Fabra, Barcelona, Spain
}

\begin{document}

\maketitle

\begin{abstract}%
  We consider an adversarial variant of the classic $K$-armed linear contextual bandit problem 
where the sequence of loss functions associated with each arm are allowed to change without 
restriction over time. Under the assumption that the $d$-dimensional contexts are generated 
i.i.d.~at random from a known distribution, we develop computationally efficient algorithms based 
on the classic \expexpexp algorithm. Our first algorithm, \linexpreal, is shown to achieve a
regret guarantee of $\wt{O}(\sqrt{KdT})$ over $T$ rounds, which matches the best known lower 
bound for this problem. Our second algorithm, \linexprobust, is shown to be robust to misspecification, 
in that it achieves a regret bound of $\wt{O}((Kd)^{1/3}T^{2/3}) + \varepsilon \sqrt{d} T$ if the 
true reward function is linear up to an additive nonlinear error uniformly bounded in absolute 
value by $\varepsilon$. To our knowledge, our performance guarantees constitute the very first 
results on this problem setting.
\end{abstract}


\section{Introduction}\label{sec:intro}
The contextual bandit problem is one of the most important sequential decision-making problems 
studied in the machine learning literature. Due to its ability to account for contextual 
information, the applicability of contextual bandit algorithms is far superior to that of standard 
multi-armed bandit methods: the framework of contextual bandits can be used to address a 
broad range of important and challenging real-world decision-making problems such as sequential 
treatment allocation \citep{TewariM17}  and online advertising \citep{LCLS10}. 
On the other hand, the framework is far less complex than that of general 
reinforcement learning, which allows for proving formal performance guarantees under relatively 
mild assumptions. As a result, there has been significant interest in this problem within the 
learning-theory community, resulting in a wide variety of algorithms with performance guarantees 
proven under a number of different assumptions. In the present paper, we fill a 
gap in this literature and design computationally efficient algorithms with strong 
performance guarantees for an adversarial version of the \emph{linear contextual bandit} problem.

Perhaps the most well-studied variant of the contextual bandit problem is that of \emph{stochastic 
linear contextual bandits} \citep{Aue02, RT10, Lihong1, NIPS2011_4417, LSz17}. 
First proposed by \citet{AL99}, this version supposes that the loss 
of each action is a fixed linear function of the vector-valued context, up to some zero-mean noise. 
Most algorithms designed for this setting are based on some variation of the ``optimism in the face 
of uncertainy'' principle championed by \citet{Aue02,auer2002finite}, or more generally by an 
appropriate exploitation of the concentration-of-measure phenomenon \citep{BoLuMa13}.
By now, this problem setting is very well-understood in many respects: there exist several 
computationally efficient, easy-to-implement algorithms achieving near-optimal worst-case 
performance guarantees \citep{NIPS2011_4417,AG13b}. These algorithms can be even
adapted to more involved loss models like generalized linear models, Gaussian processes, or 
very large structured model classes while retaining their performance guarantees 
\citep{filippi10genlin,SKKS09,CCLVR19,FKL19}. 
That said, most algorithms for stochastic linear contextual 
bandits suffer from the limitation that they are sensitive to \emph{model misspecification}: their 
performance guarantees become void as soon as the true loss functions deviate from the postulated 
linear model to the slightest degree. This issue has very recently attracted quite some attention 
due to the work of \citet{DKWY19}, seemingly implying that learning an $\varepsilon$-optimal policy 
in a contextual bandit problem has an extremely large sample complexity when assuming that the 
linear model is $\varepsilon$-inaccurate (defined formally later in our paper). This claim 
was quickly countered by \citet{VRD19} and \citet{LSz19}, who both showed that learning a 
(somewhat worse) $\varepsilon\sqrt{d}$-optimal policy is feasible with the very same sample 
complexity as learning a near-optimal policy in a well-specified linear model. Yet, since 
algorithms that are currently known to enjoy these favorable guarantees are quite complex, 
there is much work left to be done in designing practical algorithms with strong 
guarantees under model misspecification. This is one of the main issues we address in this paper.

Another limitation of virtually all known algorithms for linear contextual bandits is that they 
crucially rely on assuming that the loss function is \emph{fixed during the learning procedure}
\footnote{Or make other stringent assumptions about the losses, such as supposing that their total variation is 
bounded---see, e.g., \citet{CSZ19,RVC19,kim2019nearoptimal}.}.
This is in stark contrast with the literature on multi-armed (non-contextual) bandits, where there 
is a rich literature on both stochastic bandit models assuming i.i.d.~rewards and adversarial 
bandit models making no assumptions on the sequence of loss functions---see \citet{bubeck12survey} 
and \citet{LSz19book} for an excellent overview of both lines of work. Our main contribution in the 
present paper is addressing this gap by designing and analyzing algorithms that are guaranteed to 
work for arbitrary sequences of loss functions. 
While it is tempting to think that the our bandit problem can be directly 
addressed by a minor adaptation of algorithms developed for adversarial linear bandits, this is 
unfortunately not the case: all algorithms developed for such problems require a \emph{fixed 
decision set}, whereas reducing the linear contextual bandit problem to a linear bandit problem 
requires the use \emph{decision sets that change as a function of the contexts} 
\citep[Section~18]{LSz19book}. As a crucial step in our analysis, we will
assume that the contexts are generated in an i.i.d.~fashion and that the loss function in each round 
is statistically independent from the context in the same round. This assumption will allow us to 
relate the contextual bandit problem to a set of auxiliary bandit problems with a 
\emph{fixed action sets}, and reduce the scope of the analysis to these auxiliary 
problems.

Our main results are the following. We consider a $K$-armed linear contextual bandit problem with $d$-dimensional contexts where in each round, a loss function mapping contexts and actions to real numbers is chosen by an adversary in a sequence of $T$ rounds. The aim of the learner is to minimize its regret, defined as the gap between the total incurred by the learner and that of the best decision-making policy $\pi^*$ fixed in full knowledge of the loss sequence. 
We consider two different assumptions on the loss function. 
Assuming that the loss functions selected by the adversary are all linear, we propose an algorithm achieving a regret bound of order $\sqrt{KdT}$, which is known to be minimax optimal even in the simpler case of i.i.d.~losses (cf.~\citealp{Lihong1}). 
Second, we consider loss functions that are ``nearly linear'' up to an additive nonlinear function uniformly bounded by 
$\varepsilon$. For this case, we design an algorithm that guarantees regret bounded by $(Kd)^{1/3}T^{2/3} + \varepsilon 
\sqrt{d} T$. Notably, these latter bounds hold against \emph{any} class of policies and the $\varepsilon \sqrt{d} T$ 
overhead paid for nonlinearity is optimal when $K$ is large \citep{LSz19}. Both algorithms are computationally 
efficient, but require some prior knowledge to the distribution of the contexts.

There exist numerous other approaches for contextual bandit problems that do not rely on modeling 
the loss functions, but rather make use of a class of \emph{policies} that map contexts to actions. 
Instead of trying to fit the loss functions, these approaches aim to identify the best policy in 
the class. A typical assumption in this line of work is that one has access to a computational 
oracle that can perform various optimization problems over the policy class (such as returning an 
optimal policy given a joint distribution of context-loss pairs for each action). Given access to 
such an oracle, there exist algorithms achieving near-optimal performance guarantees when the loss 
function is fixed \citep{DHKKLRZ11,AHKLLS14}. More relevant to our present work are the works of \citet{RS16} 
and \citet{SKS16,SLKS16} who propose efficient algorithms with guaranteed performance for 
adversarial loss sequences and i.i.d.~contexts. Unlike the algorithms we present in this paper, 
these methods fail to guarantee optimal performance guarantees of order $\sqrt{T}$. 
Yet another line of work considers optimizing surrogate losses, where achieving regret of order 
$\sqrt{T}$ is indeed possible, with the caveat that the bounds only hold for the surrogate loss \citep{KSST08,BOZ17,FK18}.

The rest of the paper is organized as follows. After defining some basic notation, Section~\ref{sec:prelim} presents our problem formulation and states our assumptions. We present our algorithms and main results in Section~\ref{sec:main} and provide the proofs in Section~\ref{sec:analysis}. Section~\ref{sec:conc} concludes the paper by discussing some implications of our results and posing some open questions for future study.

\paragraph{Notation.} We use $\iprod{\cdot}{\cdot}$ to denote inner products in Euclidean space and by $\twonorm{\cdot}$ we denote the Euclidean norm. 
For a symmetric positive semidefinite matrix $A$, we use $\lambdamin(A)$ to denote its smallest eigenvalue. We use $\opnorm{A}$ to denote the operator norm of $A$ and we write $\trace{A}$ for the trace of a matrix $A$. Finally, we use $A \succcurlyeq 0$ to denote that an operator A is positive semi-definite, and we use $A \succcurlyeq B$ to denote $A - B \succcurlyeq 0$.

\section{Preliminaries}
\label{sec:prelim}
We consider a sequential interaction scheme between a \emph{learner} and its 
\emph{environment}, where the following steps are repeated in a sequence of rounds $t=1,2,\dots,T$:
\begin{enumerate}
 \item For each action $a=1,2,\dots,K$, the environment chooses a loss vector 
$\theta_{t,a}\in\real^d$,
 \item independently of the choice of loss vectors, the environment draws the context vector 
$X_t\in\real^d$ from the context distribution $\Dw$, and reveals it to the learner,
 \item based on $X_t$ 
 and possibly some randomness, the learner chooses action $A_t\in[K]$,
 \item the learner incurs and observes loss $\ell_{t}(X_t,A_t) = \iprod{X_t}{\theta_{t,A_t}}$.
\end{enumerate}
The goal of the learner is to pick its actions in a way that its total loss is as small as 
possible. Since we make no statistical assumptions about the sequence of losses (and in fact 
we allow them to depend on all the past interaction history), the learner cannot 
actually hope to incur as little loss as the best sequence of actions. A more reasonable aim is to 
match the performance of the \emph{best fixed policy} that maps contexts to actions in a static 
way. Formally, the learner will consider the set $\Pi$ of all policies $\pi:\real^d\ra[K]$, and aim 
to minimize its \emph{total expected regret} (or, simply, \emph{regret}) defined as
\[
 R_T = \max_{\pi\in\Pi} \EE{\sum_{t=1}^T \bpa{\ell_{t}(X_t,A_t) - \ell_{t}(X_t,\pi(X_t))}} = 
 \max_{\pi\in\Pi} \EE{\sum_{t=1}^T \iprod{X_t}{\theta_{t,A_t} - \theta_{t,\pi(X_t)}}},
\]
where the expectation is taken over the randomness injected by the learner, as well as the 
sequence of random contexts. For stating many of our technical results, it will be useful to define 
the filtration $\mathcal{F}_t = \mathcal{\sigma} ( X_s,  A_s , \forall s\le t )$ and the  notations
$\EEt{ \cdot} = \EE{ \cdot| \mathcal{F}_{t-1}}$ and $\PPt{ \cdot} = \PP{ \cdot| \mathcal{F}_{t-1}}$.
We will also often make use of a \emph{ghost sample} $X_0\sim \Dw$ drawn independently from the 
entire interaction history $\F_T$ for the sake of analysis. For instance, we can immediately show 
using this technique that for any policy $\pi$, we have 
\[
\EE{\iprod{X_t}{\theta_{t,\pi(X_t)}}} = \EE{\EEt{\iprod{X_t}{\theta_{t,\pi(X_t)}}}} = 
\EE{\EEt{\iprod{X_0}{\theta_{t,\pi(X_0)}}}} = \EE{\iprod{X_0}{\EE{\theta_{t,\pi(X_0)}}}},
\]
where the last expectation emphasizes that the loss vector $\theta_{t,a}$ may depend on the past 
random contexts and actions. 
This in turn can be used to show
\begin{align*}
 \EE{\sum_{t=1}^T \iprod{X_t}{\theta_{t,\pi(X_t)}}}  = \EE{\sum_{t=1}^T 
\iprod{X_0}{\EE{\theta_{t,\pi(X_0)}}}} \ge  \EE{\min_a \sum_{t=1}^T \iprod{X_0}{\EE{\theta_{t,a}}}},
\end{align*}
so the optimal policy $\pi_T^*$ that the learner compares itself to is the one defined through the 
rule
\begin{equation}\label{best_policy}
	\pi_T^*(x) = \argmin_a \sum_{t=1}^T \iprod{x}{\EE{\theta_{t,a}}}\qquad\qquad (\forall x\in\real^d).
\end{equation}
We will refer to policies of the above form as \emph{linear-classifier policies} and are defined 
through the rule $\pi_\theta(x) = \argmin_{a} \iprod{x}{\theta_a}$ for any collection of 
parameter vectors $\theta\in\real^{K\times d}$. We will also rely on the notion of \emph{stochastic 
policies} that assign probability distributions over the action set to each state, and use $\pi(a|x)$ to denote 
the probability that the stochastic policy $\pi$ takes action $a$ in state $x$.

Our analysis will rely on the following assumptions. We will suppose 
the context distribution is supported on the bounded set $\X$ with each $x\in \X$ satisfying 
 $\twonorm{x} \le \sigma $ for some $\sigma>0$, and also that $\twonorm{\theta_{t,a}} \le R$ for some positive $R$ for all $t,a$. 
Additionally, we suppose that the loss function is bounded by one in absolute value:
  $\big|\loss_{t}(x,a) \big| \le 1$ for all $t$, $a$ and all $x \in \X$. 
We will finally assume that the covariance matrix of the contexts $\Sigma = \EE{X_tX_t\transpose}$ is positive 
definite with its smallest eigenvalue being $\lambdamin > 0$. 

\section{Algorithms and main results}\label{sec:main}
Our main algorithmic contribution is a natural adaptation of the classic 
\expexpexp algorithm of \citet{auer2002finite} to the linear contextual bandit setting. 
The key idea underlying our method is to design a 
suitable estimator of the loss vectors and use these estimators to define a policy for the learner 
as follows: letting $\htheta_{t,a}$ be an estimator of the true loss vector $\theta_{t,a}$ and 
their cumulative sum $\hTheta_{t,a} = \sum_{k=1}^t \htheta_{k,a}$, our algorithm will base its 
decisions on the values $\langle X_t, \hTheta_{t-1,a}\rangle$ serving as estimators of 
the cumulative losses $\iprod{X_t}{\Theta_{t-1,a}} = \sum_{k=1}^{t-1} \iprod{X_t}{\theta_{k,a}}$. The algorithm then uses these values in an 
exponential-weights-style algorithm and plays action $a$ with probability proportional to 
$\exp\bpa{-\eta \langle X_t, \hTheta_{t-1,a}\rangle}$, where $\eta>0$ is a 
\emph{learning-rate} 
parameter. We present a general version of this method as Algorithm~\ref{alg:linexp3}. As a tribute to
the \linucb algorithm, a natural extension of the classic UCB algorithm to linear 
contextual bandits, we refer to our algorithm as \linexp.

\begin{algorithm}[H]
	\caption{\linexp}
	\label{alg:linexp3}
	\textbf{Parameters:} Learning rate $\eta>0$, exploration parameter $\gamma \in (0,1)$, $\Sigma $\\
	\textbf{Initialization:} Set $\theta_{0,i} = 0$ for all $i\in[K]$.
	\\
	\textbf{For} $t = 1, \dots, T$, \textbf{repeat:}
	\begin{enumerate}
		\item Observe $X_t$ and, for all $a$, set 
		\[ 
            w_{t}(X_t,a) = \exp\pa{- \eta  \sum_{s=0}^{t-1} \langle X_t, \htheta_{s,a}\rangle},
        \]
		\item draw $A_t$ from the policy defined as
		\[
            \pi_{t}\left(a\middle|X_t\right)= (1-\gamma) \frac{w_{t}(X_t,a)}{\sum_{a'} 
w_{t}(X_t,a')} + \frac{\gamma}{K},
		\]
		\item observe the loss $\loss_t(X_t,A_t)$ and compute $\htheta_{t,a}$ for all $a$.
	\end{enumerate}
\end{algorithm}
As presented above, \linexp is more of a template than an actual algorithm since it does not 
specify the loss estimators $\htheta_{t,a}$. Ideally, one may want to use \emph{unbiased} 
estimators that satisfy $\EEb{\htheta_{t,a}} = \theta_{t,a}$ for all $t,a$. Our key contribution is 
designing two different (nearly) unbiased estimators that will allow us to prove performance 
guarantees of two distinct flavors. Both estimators are efficiently computable, but require
some prior knowledge the context distribution $\Dw$. In what follows, we describe the two variants of \linexp based on 
the two estimators and 
state the corresponding performance guarantees, and relegate the proof sketches to 
Section~\ref{sec:analysis}. We also present two simple variants of our algorithms that work with 
various degrees of full-information feedback in Appendix~\ref{sec:fullinfo}.

\subsection{Algorithm for nearly-linear losses: \linexprobust}\label{sec:alg_robust}
We begin by describing the simpler one of our two algorithms, which will be seen to be robust to 
misspecification of the linear loss model. We will accordingly refer to this algorithm as 
\linexprobust. Specifically, we suppose in this section
that $\loss_t(x,a) = \iprod{x}{\theta_{t,a}} + \varepsilon_{t}(x,a)$, where $\varepsilon_{t}(x,a): 
 \real^d \times K  \to \real$ is an arbitrary nonlinear function satisfying $|\varepsilon_{t}(x,a)| \le \varepsilon$ for 
all $t$, $x$ and $a$ and some $\varepsilon>0$.
Also supposing that we have perfect knowledge of the covariance matrix $\Sigma$,
we define the loss estimator used by \linexprobust for all actions $a$ as
\begin{equation}\label{eq:robustest}
 \htheta_{t,a} = \frac{\II{A_t = a}}{\pi_t(a|X_t)} \Sigma^{-1} X_t \loss_t(X_t,A_t).
\end{equation}
In case the loss is truly linear, it is easy to see that the above is an unbiased estimate since
\begin{align*}
 \EEt{\htheta_{t,a}} &= \EEt{\EEcct{\frac{\II{A_t = a}}{\pi_t(a|X_t)} \Sigma^{-1} X_t 
\iprod{X_t}{\theta_{t,a}}}{X_t}} = \EEt{\EEcct{\frac{\II{A_t = a}}{\pi_t(a|X_t)}}{X_t} \Sigma^{-1} 
X_t X_t\transpose \theta_{t,a}}
\\
&= \EEt{\Sigma^{-1} X_t X_t\transpose \theta_{t,a}} = \theta_{t,a},
\end{align*}
where we used the definition of $\Sigma$ and the independence of $\theta_{t,a}$ from $X_t$ in the 
last step. 
A key result in our analysis will be that, for nonlinear losses, the estimate above satisfies
\[
 \left|\EEt{\biprod{X_t}{\htheta_{t,a}} - \loss_t(X_t,a)}\right| \le \varepsilon\sqrt{d}.
\]
Our main result regarding the performance of \linexprobust is the following:
\begin{theorem}\label{th_robust} 
	For any positive $\eta \le  \frac{\gamma \lambdamin}{ K \sigma^2}$ and for any $\gamma \in (0,1)$ the expected 
	regret of \linexprobust  satisfies
	\begin{align*}
	R_T \le 2\sqrt{d}\varepsilon T + 2\gamma T +  \frac{2\eta KdT}{\gamma} +\frac{\log K}{\eta}. 
	\end{align*}
	Furthermore, letting
	$\eta = T^{-2/3}\pa{Kd}^{-1/3} \pa{\log K}^{2/3}$, $\gamma = 
T^{-1/3}\pa{Kd\log K}^{1/3}$ and supposing that  $T$ is large enough so that 
$\eta \le \frac{\gamma \lambdamin}{K\sigma^2}$ holds, the expected regret of \linexprobust  satisfies
	\begin{align*}
	R_T \le 5 T^{2/3}\pa{Kd\log K}^{1/3} + 2\varepsilon\sqrt{d} T. 
	\end{align*}
\end{theorem}

\subsection{Algorithm for linear losses: \linexpreal}\label{sec:alg_real}
Our second algorithm uses a more sophisticated estimator based on the 
covariance matrix
\[
 \Sigma_{t,a} = \EEt{\II{A_t=a} X_tX_t\transpose},
\]
which is used to define the estimate
\[
 \ttheta^*_{t,a} = \II{A_t = a} \Sigma_{t,a}^{-1} X_t \iprod{X_t}{\theta_{t,a}}.
\]
This can be easily shown to be unbiased as
\[
 \EEt{\ttheta^*_{t,a}} = \EEt{\II{A_t = a} \Sigma_{t,a}^{-1} X_t \iprod{X_t}{\theta_{t,a}}} 
 = \EEt{\Sigma_{t,a}^{-1} \II{A_t = a} X_tX_t\transpose \theta_{t,a}} = \theta_{t,a},
\]
where we used the conditional independence of $\theta_{t,a}$ and $X_t$ once again. Unfortunately, 
unlike the estimator used by \linexprobust, the bias of this estimator cannot be bounded when the 
losses are misspecified. However, its variance turns out to be much smaller for well-specified 
linear losses, which will enable us to prove tighter regret bounds for this case.

One downside of the estimator defined above is that it is very difficult to compute: the 
matrix $\Sigma_{t,a}$ depends on the joint distribution of the context $X_t$ and the action $A_t$, 
which has a very complicated structure. 
While it is trivially easy to design an unbiased estimator of 
$\Sigma_{t,a}$, it is very difficult to compute a reliable-enough estimator of its inverse. To 
address this issue, we design an alternative estimator based on a matrix generalization of the 
Geometric Resampling method of \citet{NeuBartok13,NB16}. The method that we hereby dub 
\emph{Matrix Geometric Resampling} (MGR) has two parameters $\beta>0$ and $M\in \mathbb{Z_+}$, and
constructs an estimator of $\Sigma_{t,a}^{-1}$ through the following procedure:

\vspace{.25cm}
\makebox[\textwidth][c]{
\fbox{
\begin{minipage}{.6\textwidth}
\textbf{Matrix Geometric Resampling}
\vspace{.1cm}
\hrule
\vspace{.1cm}
\textbf{Input:} data distribution $\Dw$, policy $\pi_t$, action $a$.\\
\textbf{For $k = 1, \dots, M$, repeat}:
	\begin{enumerate}
	\vspace{-2mm}
	\item Draw $X(k)\sim\Dw$ and $A(k) \sim \pi_t(\cdot|X(k))$,
	\vspace{-2mm}
	\item compute $B_{k, a} = \II{A(k) = a} X(k)X(k)\transpose $,
	\vspace{-2mm}
	\item compute $A_{k, a} = \prod_{j=1}^k  (I - \beta B_{k, a}) $.
	\end{enumerate}
	\vspace{-2mm}
\textbf{Return $\widehat{\Sigma}^{+}_{t,a} = \beta I+ \beta \sum_{k=1}^M A_{k,a}$.}
\end{minipage}
}
}
\vspace{.25cm}

\noindent Clearly, implementing the MGR procedure requires sampling access to the distribution $\mathcal{D}$. The 
rationale behind the estimator $\wh{\Sigma}^+_{t,a}$ is the following. Assuming that
$M = \infty$ and $\beta \le \frac{1}{\sigma^2}$, we can observe that $\EEt{B_{k,a}} = \Sigma_{t,a}$ 
and, due to independence of the contexts $X(k)$ from each other, 
\[
 \EEt{A_{k,a}} = \EEt{\prod_{j=1}^k  (I - \beta B_{k, a})} = \pa{I - \beta \Sigma_{t,a}}^k,
\]
we can see that $\wh{\Sigma}^+_{t,a}$ is a good estimator of $\Sigma_{t,a}^{-1}$ on expectation:
\begin{equation}\label{eq:sigmaplus}
 \EEt{\widehat{\Sigma}^{+}_{t,a}} = \beta I+ \beta \sum_{k=1}^\infty \pa{I - \beta 
\Sigma_{t,a}}^k = 
 \beta \sum_{k=0}^\infty \pa{I - \beta \Sigma_{t,a}}^k = \beta \pa{\beta \Sigma_{t,a}}^{-1} = 
\Sigma_{t,a}^{-1}.
\end{equation}
As we will see later in the analysis, the bias introduced by setting a finite $M$ can be controlled relatively easily.

Based on the above procedure, we finally define our loss estimator used in this section as
\begin{equation}\label{eq:MGRest}
\ttheta_{t,a} = \widehat{\Sigma}^{+}_{t,a} X_t \loss(X_t, A_t) \II{A_t= a}.
\end{equation}
Via a careful incremental implementation, the estimator can be computed in $O(MKd)$ time and $M$ calls to the oracle 
generating samples from the context distribution $\mathcal{D}$. We present the details of this efficient computation 
procedure in Appendix~\ref{appendixC}.
We will refer to the version of \linexp using the estimates above as \linexpreal, alluding to its 
favorable guarantees obtained for realizable linear losses.
Our main result in this section is the following guarantee regarding the performance of \linexpreal:
\begin{theorem}\label{th_real}\
	For  $\gamma\in (0,1)$, $M \ge 0$, any positive $\eta \le \frac{2}{M+1}$ and any positive $\beta \le \frac{1}{2\sigma^2}$, the expected 
	regret of \linexpreal  satisfies
	\begin{align*}
	R_T \le 2T \sigma R \cdot \exp\pa{-\frac{\gamma\beta\lambdamin M}{K}}  + 2\gamma T + 3\eta K d 
T + \frac{\log K}{\eta}.
	\end{align*}
	 Furthermore, letting $\beta = \frac{1}{2\sigma^2}$, $M = \left\lceil\frac{K \sigma^2 \log (T 
\sigma^2 R^2)}{\gamma  \lambdamin}\right\rceil$,
	$\gamma = \sqrt{\frac{\log (T \sigma^2 R^2)}{T}}$, and $\eta =\sqrt{\frac{\log K}{d K T \log (T \sigma^2 R^2)}}$ and supposing that $T$ is large enough so that the above constraints are satisfied, we also have
	\begin{align*}
	 R_T \le 4\sqrt{T}  +  \sqrt{dKT\log K}\bpa{3 + \sqrt{\log (T \sigma^2 R^2)}}.
	\end{align*}
\end{theorem}

\section{Analysis}
\label{sec:analysis}
This section is dedicated to proving our main results, Theorems~\ref{th_robust} and~\ref{th_real}. We present the analysis in a modular fashion, first proving some general facts about the algorithm template \linexp, and then treat the two variants separately in Sections~\ref{sec:robust} and~\ref{sec:real} that differ in their choice of loss estimator.

The main challenge in the contextual bandit setting is that the comparator term in the regret definition features actions that depend on the observed contexts, which is to be contrasted with the classical multi-armed bandit setting where the comparator strategy always plays a fixed action. 
The most distinctive element of our analysis is the following lemma that tackles this difficulty by essentially reducing the contextual bandit problem to a set of auxiliary online learning problems defined separately for each context $x$:

\begin{lemma}\label{context2exp3}
 Let $\pi^*$ be any fixed stochastic policy and let $X_0$ be sample from the context distribution $\Dw$ independent 
from $\F_T$. Suppose that $\pi_t\in\F_{t-1}$ and that $\EEtb{\htheta_{t,a}} = \theta_{t,a}$ for all $t,a$. Then,
 \begin{equation}\label{eq:reduction}
  \EE{\sum_{t=1}^T\sum_{a} \bpa{\pi_t(a|X_t) - \pi^*(a|X_t)} \iprod{X_t}{\theta_{t,a}} } =
  \EE{\sum_{t=1}^T\sum_{a}\bpa{ \pi_t(a|X_0) - \pi^*(a|X_0)} \biprod{X_0}{\htheta_{t,a}}} .
  \end{equation}
\end{lemma}
\begin{proof}

Fix any $t$ and $a$. Then, we have
\begin{align*} 
&\EEt{\bpa{\pi_t(a|\tX) - \pi^*(a|\tX)} \biprod{\tX}{\htheta_{t,a}}} 
= \EEt{\EEcct{\bpa{ \pi_t(a|\tX) - \pi^*(a|\tX)} \biprod{\tX}{\htheta_{t,a}}}{\tX}}  
\\
&\quad= \EEt{\EEcct{\bpa{ \pi_t(a|\tX) - \pi^*(a|\tX)} \iprod{\tX}{\theta_{t,a}}}{\tX}}  
= \EEt{\bpa{ \pi_t(a|\tX) - \pi^*(a|\tX)} \iprod{\tX}{\theta_{t,a}}}
\\
&\quad = \EEt{\bpa{ \pi_t(a|X_t) - \pi^*(a|X_t)} \iprod{X_t}{\theta_{t,a}}},
\end{align*}
where the first step uses the tower rule of expectation, the second that 
$\mathbb{E}_t\bigl[\htheta_{t,a}\big|\tX\bigr] = \theta_{t,a}$ that holds due to the independence of $\htheta_t$ and 
$\theta_t$ on $\tX$, the third step is the tower rule again, and the last step uses that $\tX$ and 
$X_t$ have the same distribution and both are conditionally independent on $\theta_t$. Summing up for all actions 
concludes the proof.
\end{proof}
Notably, the lemma above is not specific to our algorithm \linexp and only uses the properties of the loss estimator. Applying the lemma to the policies $\pi_t$ produced by \linexp and using \emph{any} comparator $\pi^*$, we can notice that the term on the right hand side is the regret $R_T$ of the algorithm. We stress here that the above result is in fact very powerful since it does not assume \emph{anything} (except measurability) about $\pi^*$, even allowing it to be non-smooth---we provide a more detailed discussion of this issue in Section~\ref{sec:conc}.
In order to interpret the term on the right-hand side of Equation~\eqref{eq:reduction}, let us consider an auxiliary 
online learning problem for a fixed $x$ with $K$ actions and losses $\hloss_{t}(x,a) = \biprod{x}{\htheta_{t,a}}$ for 
each $t,a$, and consider running a copy of the classic exponential-weights algorithm\footnote{For the sake of clarity, 
we omit the step of mixing in the uniform distribution in this expository discussion.} of \citet{LW94} fed with these 
losses. The probability distribution played by this algorithm over the actions $a$ is given as $\pi_t(a|x) \propto 
\exp\pa{-\eta \sum_{s=1}^{t-1}\hloss_{s}(x,a)}$, which implies that the regret in the auxiliary game against comparator 
$\pi^*$ at $x$ can be written as
\[
 \wh{R}_T(x) = \sum_{t=1}^T\sum_{a} \bpa{\pi_t(a|x) - \pi^*(a|x)} \biprod{x}{\htheta_{t,a}}.
\]
This brings us to the key observation that the term on the right-hand side of the equality in Lemma~\ref{context2exp3} is exactly $\EE{R_T(X_0)}$. Thus, our proof strategy will be to prove an almost-sure regret bound for the auxiliary games defined at each $x$ and take expectation of the resulting bounds with respect to the law of $X_0$, thus achieving a bound on the regret $R_T$. The following lemma provides the desired bounds for the auxiliary games:
\begin{lemma}\label{exp2base}
Fix any $x\in\X$ and suppose that $\htheta_{t,a}$ is such that $\big|\eta\biprod{x}{\htheta_{t,a}}\big| < 1 $. Then, the regret of \linexp in the auxiliary game at $x$ satisfies
 \[
\hR_T(x) \le   \frac{\log K}{\eta} +  2\gamma U_T(x) + \eta \sum_{t=1}^T  \sum_{a=1}^K    \pi_t(a|x)  \biprod{x}{\htheta_{t,a}}^2,
\]
where $U_T(x) = \sum_{t=1}^T\pa{\frac{1}{K}\sum_a		\biprod{x}{\htheta_{t,a}} - \biprod{x}{\htheta_{t,\pi^*(x)}}}$.
\end{lemma}
In the above bound, $U_T(x)$ is the regret of the uniform policy, which can be bounded by $T$ for all algorithms on 
expectation. The proof is a straightforward application of standard ideas from the classical \expexpexp analysis due to 
\citet{AuerCFS02}, and we include it in Appendix~\ref{appendixA} for completeness. 

The lemmas above suggest that all we need to do is to bound the expectation of the second-order terms on the right-hand 
side, $\EEt{\sum_{a=1}^K    \pi_t(a|X_0)  \biprod{X_0}{\htheta_{t,a}}^2}$. This, however, is not the only challenge due 
to the fact that the estimators our algorithms use are not necessarily all unbiased. Specifically, supposing that our 
estimator can  be written as $\htheta_{t,a} = \htheta^*_{t,a} + b_{t,a}$, where $\htheta^*_{t,a}$ is such that 
$\EEtb{\htheta^*_{t,a}} = \theta_{t,a}$ and $b_{t,a}$ is a bias term, we can directly deduce the following bound from 
Lemma~\ref{context2exp3}:
\begin{equation}\label{eq:rtrxbound}
R_T \le \EEb{\hR_T(X_0)} + 2\sum_{t=1}^T \max_a |\EE{ \iprod{X_t}{ b_{t,a}} }|.
\end{equation}
The rest of the section is dedicated to finding the upper bounds on the bias term above and on the expectation of the second-order term discussed right before for both estimators~\eqref{eq:robustest} and \eqref{eq:MGRest}, therefore completing the proofs of our main results, Theorems~\ref{th_robust} and ~\ref{th_real}.

\subsection{Proof of Theorem~\ref{th_robust}}\label{sec:robust}
We first consider \linexprobust which uses the estimator $\htheta_{t,a}$ defined in Equation~\eqref{eq:robustest}. While we have already shown in Section~\ref{sec:alg_robust} that the estimator is unbiased, we now consider the case where the true loss function may be nonlinear and can be written as 
$\loss_t(x,a) = \iprod{x}{\theta_{t,a}} + \varepsilon_t(x,a)$ for some nonlinear function $\varepsilon_t$ uniformly bounded on $\X$ by $\varepsilon$. Then, we can see that our estimator satisfies
\begin{align*}
\EEtb{\biprod{\tX}{\htheta_{t,a}}}  &= \EEt{ \frac{\II{A_t = a}}{\pi(a|X_t)} \tX\transpose\Sigma^{-1} X_t \bpa{\iprod{X_t}{\theta_{t,a}} + \varepsilon_{t}(X_t,a)} }\\
& = \EEt{\iprod{\tX}{\theta_{t,a}}} + \EEt{ \tX\transpose\Sigma^{-1} X_t \varepsilon_{t}(X_t,a) },
\end{align*}
and thus the bias can be bounded using the Cauchy--Schwarz inequality as
\begin{equation}\label{eq:robustbias}
 \bigg|\EEt{ \tX\transpose\Sigma^{-1} X_t \varepsilon_{t}(X_t,a) }\bigg| \le \sqrt{\EEt{ \trace{\tX \tX\transpose \Sigma^{-1} X_t X_t\transpose \Sigma^{-1} } }} \cdot \sqrt{\EEt{\pa{\varepsilon_{t}(X_t,a)}^2}} \le \sqrt{d}\varepsilon.
\end{equation}
Here, we used $\EEt{X_0X_0\transpose X_tX_t\transpose} = \Sigma^2$, which follows from the conditional independence of $X_0$ and $X_t$ and the definition of $\Sigma$, and the boundedness of $\varepsilon_t$ in the last step.
The other key component of the proof is the following bound:
	\begin{equation}
	 \begin{split}
	&\EEt{\sum_{a=1}^K    \pi_t(a|\tX)   \biprod{\tX}{\htheta_{t,a}}	^2} =  \EEt{   \sum_{a=1}^K 
\pi_t(a|\tX) \frac{\II{A_t = a} \loss_t(X_t,a)^2}{\pi_t^2(a|X_t)} \tX\transpose \Sigma^{-1} X_t 
X_t\transpose  \Sigma^{-1} X_0} \\
	&\qquad\qquad\le \EEt{   \sum_{a=1}^K \pi_t(a|\tX) \cdot \frac{K}{\gamma} \cdot  \frac{ \II{A_t 
= a} }{ \pi_t(a|X_t)}\cdot \trace{\Sigma^{-1} X_t X_t^T \Sigma^{-1} X_0 X_0\transpose }    }	 \le	
\frac{Kd}{\gamma}\label{eq:robustquad}
	 \end{split}
	\end{equation}
	where we used $\pi_t(a|X_t)\ge \frac{\gamma}{K}$ in the first inequality and the conditional independence of $X_t$ and $X_0$ in the last step.
The problem we are left with is to prove that $\eta \big| \biprod{\tX}{\htheta_{t,a}}\big| \le 1$:
\begin{align*} 
\big| \biprod{\tX}{\htheta_{t,a}}\big| &= \frac{\II{A_t = a}}{\pi_t(a|X_t)} \left|\tX\transpose \Sigma^{-1} X_t\right| \loss_{t}(X_t,A_t)  \le \frac{K\sigma^2}{\gamma\lambdamin},
\end{align*}
where we used the conditions $\pi_t(a|X_t)\ge \frac{\gamma}{K}$ and $|\loss_{t}(x,a)|\le 1$ and the Cauchy--Schwarz inequality to show $\left|\tX\transpose \Sigma^{-1} X_t\right| \le \frac{\sigma^2}{\lambdamin}$. Having satisfied its condition, we may now appeal to Lemma~\ref{exp2base}, and the proof is concluded by combining and Equations~\eqref{eq:rtrxbound}, \eqref{eq:robustbias}, and \eqref{eq:robustquad}. 

\subsection{Proof of Theorem~\ref{th_real}}\label{sec:real}
We now turn to analyzing \linexpreal which uses the slightly more complicated loss estimator $\ttheta_{t,a}$ defined to the MGR procedure. Although we have already seen in Section~\ref{sec:alg_real} that MGR could result in an unbiased estimate if we could set $M=\infty$. However, in order to keep computation at bay, we need to set $M$ to be a finite (and hopefully relatively small) value. Following the same steps as in Equation~\eqref{eq:sigmaplus}, we can show
\[
\EEt{\widehat{\Sigma}^{+}_{t,a}} = 
 \beta \sum_{k=0}^M \pa{I - \beta \Sigma_{t,a}}^k = \Sigma_{t,a}^{-1} - (I - \beta \Sigma_{t,a})^M\Sigma_{t,a}^{-1}.
\]
Combining this insight with the definition of $\ttheta_{t,a}$ and using some properties of our algorithm, we can prove the following useful bound on the bias of the estimator:
\begin{lemma}
	\label{trace_M}
	Suppose that $M \ge  \frac{K \sigma^2 \log T}{\gamma  \lambdamin}$, $\beta = \frac{1}{2\sigma^2}$. Then, $\bigl|\EEtb{\biprod{X_t}{\theta_{t,a} - \ttheta_{t,a}}}\bigr| \le \frac{\sigma R}{\sqrt{T}}$.
\end{lemma}
\begin{proof}
We first observe that the bias of $\ttheta_{t,a}$ can be easily expressed as
\begin{align*}
	\EEtb{ \ttheta_{t,a} }
	&= \EEt{\widehat{\Sigma}^{+}_{t,a} X_t X_t\transpose \theta_{t,a} \II{A_t= a} }  = \EEt{\widehat{\Sigma}^{+}_{t,a} } \EEt{X_t X_t\transpose \II{A_t= a} }   \theta_{t,a} =\EEt{\widehat{\Sigma}^{+}_{t,a} } \Sigma_{t,a} \theta_{t,a} 
	\\
	&= \theta_{t,a} - (I - \beta \Sigma_{t,a})^M\theta_{t,a},
\end{align*}
where we have used our expression for $\EEtb{\widehat{\Sigma}^{+}_{t,a}}$ derived above.
Thus, the bias is bounded as
\[ 
\left|\EEt{X_t\transpose (I - \beta \Sigma_{t,a})^M \theta_{t,a} }\right| \le \twonorm{X_t}\cdot\twonorm{\theta_{t,a}}\opnorm{(I - \beta \Sigma_{t,a})^M}.
\]
In order to bound the last factor above, observe that $\Sigma_{t,a} \succcurlyeq \frac{\gamma}{K}\Sigma$ due to the uniform exploration used by \linexp, which implies that
\[
\opnorm{(I - \beta \Sigma_{t,a})^M}\le \pa{1 - \frac{\gamma\beta \lambdamin}{K}}^M
 \le \exp\pa{ -\frac{\gamma \beta}{K} \lambdamin M} \le \frac{1}{\sqrt{T}},
\]
where the second inequality uses $1-z\le e^{-z}$ that holds for all $z$, and the last step uses our condition on $M$. This concludes the proof.
\end{proof}
The other key term in the regret bound is bounded in the following lemma:
\begin{lemma}
Suppose that $X_t$ is satisfying $\twonorm{X_t} \le \sigma $, $0 < \beta \le \frac{1}{2\sigma^2}$ and $M > 0$. Then for each $t$, \linexpreal guarantees 
	\label{quadratic} 
	$$\EEt{\sum_{a=1}^K    \pi_t(a|\tX)   \biprod{\tX}{\ttheta_{t,a}}^2} \le 3Kd .$$
\end{lemma}
Unfortunately, the proof of this statement is rather tedious, so we have to relegate it to Appendix~\ref{appendixB}.
As a final step, we need to verify that the condition of Lemma~\ref{exp2base} is satisfied, that is, that $\eta \big|\biprod{\tX}{\ttheta_{t,a}}\big| < 1 $ holds. To this end, notice that
\begin{align*}
\eta\cdot \big| \biprod{\tX}{\ttheta_{t,a}}\big| 
&= \eta \cdot \big| \tX\transpose  \widehat{\Sigma}^{+}_{t,a} X_t \iprod{X_t}{\theta_{t,a}}  \II{A_t= a} \big| 
\le \eta \cdot \big| \tX\transpose  \widehat{\Sigma}^{+}_{t,a} X_t  \big|
\\
&\le \eta \sigma^2 \opnorm{\widehat{\Sigma}^{+}_{t,a}} \le 
\eta \sigma^2 \beta\pa{1+  \sum_{k=1}^M \opnorm{A_{k,a}}} \le \eta (M+1) /2,
\end{align*}
where we used the fact that our choice of $\beta$ ensures that $\opnorm{A_{k,a}} = \bigl\|\prod_{j=0}^k(I-\beta 
B_{j,a})\bigr\|_{\text{op}} \le 1$. 
Thus, the condition $\eta \le 2/(M+1)$ allows us to use Lemma~\ref{exp2base}, so we can conclude the proof of 
Theorem~\ref{th_real} by applying Lemma~\ref{trace_M}, Lemma~\ref{quadratic} and the bound of 
Equation~\eqref{eq:rtrxbound}.

\section{Discussion}\label{sec:conc}
Our work is the first to address the natural adversarial variant of the widely popular framework 
of linear contextual bandits, thus filling an important gap in the literature. Our 
algorithm \linexpreal achieves the optimal regret bound of of order $\sqrt{KdT}$  
and runs in time polynomial in the relevant problem parameters. To our knowledge, 
\linexpreal is the first computationally efficient algorithm to achieve near-optimal regret bounds 
in an adversarial contextual bandit problem, and is among the first ones to achieve any regret 
guarantees at all for an infinite set of policies (besides results on learning with surrogate losses, 
cf.~\citealp{FK18}). 
In the case of misspecified loss functions, our algorithm \linexprobust achieves a regret guarantee of order 
$(Kd)^{1/3}T^{2/3} + \varepsilon \sqrt{d} T$.

Whether or not the overhead of $\varepsilon \sqrt{d} T$ can be improved is presently unclear: while \citet{LSz19} 
proved that the dependence on $\sqrt{d}$ is inevitable even in the stochastic linear bandit setting when $K$ is large 
(say, order of $T$), the very recent work of \citet{foster2020ucb} shows that the overhead can be reduced to 
$\varepsilon \sqrt{K}T$ in the same setting. These results together suggest that the regret bound $\sqrt{KdT} + 
\varepsilon \sqrt{\min\ev{K,d}} T$ is achievable in for stochastic linear contextual bandits. Whether such guarantees 
can be achieved in the more challenging adversarial setting we considered in this paper remains an interesting open 
problem.

The reader may be curious if it is possible to remove the i.i.d.~assumption that we make about the 
contexts. Unfortunately, it can be easily shown that no learning algorithm can achieve sublinear 
regret if the contexts and losses are both allowed to be chosen by an adversary. To see this, we 
observe that one can embed the problem of online learning of one-dimensional threshold classifiers 
into our setting, which is known to be impossible to learn with sublinear regret 
\citep{BDPSS09,SKS16}. While one can conceive other assumptions on the contexts that make the 
problem tractable, such as assuming that the entire sequence of contexts is known ahead of time (the 
so-called \emph{transductive setting} studied by \citealp{SKS16}), such assumptions may end up being 
a lot more artificial than our natural i.i.d.~condition. In addition, it is unclear what the best 
achievable performance bounds in such alternative frameworks actually are. In 
contrast, the regret bounds we prove for \linexpreal are essentially minimax optimal. 

Our algorithm design and analysis introduces a couple of new techniques that could be of 
more general interest. First, a key element in our analysis is introducing a set of auxiliary bandit 
problems for each context $x$ and relating the regrets in these problems to the expected regret in 
the contextual bandit problem (Lemma~\ref{context2exp3}). 
While this lemma is stated in terms of 
linear losses, it can be easily seen to hold for general losses as long as one can construct 
unbiased estimates of the entire loss function. 
In this view, our algorithms can be seen as the first instances of a new family of
contextual bandit methods that are based on estimating the loss functions rather than working with 
a policy class. 
An immediate extension of our approach is to assume that the loss functions belong to 
a reproducing kernel Hilbert space and define suitable kernel-based estimators analogously to our 
estimators---a widely considered setting in the literature on stochastic contextual bandits 
\citep{SKKS09,BEL17,CCLVR19}. We also remark that our technique used to prove
Lemma~\ref{context2exp3} is similar in nature to the reduction of stochastic sleeping bandit problems
to static bandit problems used by \citet{kanade09sleeping,NV14}.

A second potentially interesting algorithmic trick 
we introduce is the Matrix Geometric Resampling for estimating inverse covariance matrices. While 
such matrices are broadly used for loss estimation in the literature on adversarial linear 
bandits \citep{McMaBlu04,AweKlein04,DHK08,audibert13regret}, the complexity of computing them never seems to be discussed in the 
literature. Our MGR method provides a viable option for tackling this problem. For the curious 
reader, we remark that the relation between the iterations defining MGR and the dynamics of gradient 
descent for linear least-squares estimation is well-known in the stochastic optimization literature, 
where SGD is known to implement a spectral filter function approximating the inverse covariance 
matrix \citep{RM51,GyW96,BM13,NR18}.

Besides the most important question of whether or not our guarantees for the misspecified setting can be improved, we 
leave a few more questions open for further investigation. One limitation of our methods is that they require prior 
knowledge of the context distribution $\Dw$. We conjecture that it may be possible to overcome this limitation by 
designing slightly more sophisticated algorithms that estimate this distribution from data. Second, it appears to be an 
interesting challenge to prove versions of our performance guarantees that hold with high probability by using 
optimistically estimators as done by \citet{BLLRS11,Neu15b}, or if data-dependent bounds depending on the total loss of 
the best expert rather than $T$ can be achieved in our setting \citep{AKLLS17,AZBL18,Neu15}. We find it likely that such 
improvements are possible at the expense of a significantly more involved analysis.

\acks{We would like to thank Haipeng Luo, Chen-Yu Wei and Chung-Wei Lee for pointing out a technical issue with an earlier version of our proof of Lemma 6, and we thank Wojciech Kot\l owski for his help with the updated proof of the same lemma. We thank the three anonymous reviewers for their valuable feedback that helped us improve the paper. G.~Neu was 
supported by ``la Caixa'' Banking Foundation through the Junior Leader Postdoctoral Fellowship Programme, a Google 
Faculty Research Award, and a Bosch AI Young Researcher Award.}

\bibliographystyle{abbrvnat}
\bibliography{contextual,shortconfs}

\appendix

\section{Proof of Lemma~\ref{exp2base}}
\label{appendixA}
	The proof follows the standard analysis of \expexpexp originally due to \citet{AuerCFS02}.
	We begin by recalling the notation $w_t(x,a) = \exp\bpa{-\eta \sum_{s=1}^{t-1}\biprod{x}{\htheta_{t,a}}}$ and introducing $W_t(x) = \sum_{a=1}^{K} w_t(x,a) $. 
	The proof is based on analyzing $\log W_{T+1}(x)$, which can be thought of as a potential function in terms of the cumulative losses.
	We first observe that $\log W_{T+1}(x)$ can be lower-bounded in terms of the cumulative loss:
	\begin{align*}
	\log \bigg(  \frac{W_{T+1}(x)  }{  W_{1}(x)   }  \bigg) \ge \log \bigg(  \frac{w_{ T+1}(x,\pi^*(x))  }{  W_{1}(x)   }  \bigg)
	 = - \eta \sum_{t=1}^T x\transpose \htheta_{t, \pi^*(x)} -  \log K.
	\end{align*}
	On the other hand, for any $t$, we can prove the upper bound 
	\begin{align*}
	\log \frac{W_{t+1}(x)}{W_{t}(x)} &=   \log \bigg(  \sum_{a=1}^K \frac{w_{t+1}(x,a)}{W_t(x)} \bigg)  
	= \log \bigg(  \sum_{a=1}^K \frac{w_{ t}(x,a) e^{-\eta  
		\siprod{x}{\htheta_{t,a}} }}{W_t(x)} \bigg) 
    \\
    &=
	\log \bigg(  \sum_{i=1}^K\frac{\pi_t(a|x)  - \gamma/K}{1 - \gamma}  \cdot e^{-\eta  
		\siprod{x}{\htheta_{t,a}} } \bigg) 
	\\
    &\stackrel{(a)}{\le}
	\log \bigg(  \sum_{i=1}^K\frac{\pi_t(a|x)  - \gamma/K}{1 - \gamma}  \bigg( 1 - \eta  
	\biprod{x}{\htheta_{t,a}} + \big(\eta  \biprod{x}{\htheta_{t,a}}  
	\big)^2   \bigg) \bigg) 
	\\
	&\stackrel{(b)}{\le} \sum_{a=1}^K \frac{\pi_t(a|x)}{1 - \gamma}   \bigg(  - \eta  
	\biprod{x}{\htheta_{t,a}} + \big(\eta  \biprod{x}{\htheta_{t,a}}  
	\big)^2   \bigg) +  \frac{\eta\gamma}{K(1-\gamma)} \sum_a \biprod{x}{\htheta_{t,a}},
	\end{align*}
	where in step $(a)$ we used the inequality $e^{-z} \le 1 - z + z^2$, which holds for $z\ge -1$, and in step $(b)$ we used the inequality $\log(1+z ) \le z$ that holds for any $z$.

	Noticing that $\sum_{t=1}^T \log \frac{W_{t+1}}{W_t} = \log \frac{W_{T+1}}{W_1}$, we can sum both sides of the above inequality for all $t=1,\dots,T$ and compare with the lower bound to get
	\begin{align*}
		- \eta \sum_{t=1}^T x\transpose \htheta_{t, \pi^*(x)} -  \ln K \le \sum_{t=1}^{T}\sum_{a=1}^K \frac{\pi_t(a|x)}{1 - \gamma}   \bigg(  - \eta  
		\biprod{x}{\htheta_{t,a}} + \big(\eta  \biprod{x}{\htheta_{t,a}}  
		\big)^2   \bigg) +  \frac{\eta\gamma\sum_a \biprod{x}{\htheta_{t,a}}}{K(1-\gamma)} .
	\end{align*}
	Reordering and multiplying both sides by $\frac{1-\gamma}{\eta}$ gives
	\begin{align*}
		&\sum_{t=1}^{T}\pa{\sum_{a=1}^K \pi_t(a|x) \biprod{x}{\htheta_{t,a}} - \biprod{x}{\htheta_{t, \pi^*(x)}}}  
		\\
		&\qquad\qquad\le \frac{(1-\gamma)\log K}{\eta} + \eta \sum_{t=1}^{T}\sum_{a=1}^K \pa{\siprod{x}{\htheta_{t,a}}}^2    +  \gamma \sum_{t=1}^T\pa{\frac{1}{K}\sum_a		\biprod{x}{\htheta_{t,a}} - \biprod{x}{\htheta_{t,\pi^*(x)}}}.
	\end{align*}
	This concludes the proof.
\jmlrQED

\section{Proof of Lemma~\ref{quadratic}}
\label{appendixB}
The proof relies on a series of matrix operations, and makes repeated use of the following identity that holds for any symmetric positive definite matrix $S$:
\[
\sum_{k=0}^M \pa{I - S}^k = S^{-1} - (I - S)^MS^{-1}.
\]
We start by plugging in the definition of $\ttheta_{t,a}$ and writing
	\begin{align*}
	\EEt{\sum_{a=1}^K    \pi_t(a|\tX)   \biprod{\tX}{\ttheta_{t,a}}^2} &= \EEt{\sum_{a = 1}^K \pi_t(a|\tX)   \pa{\tX\transpose \Sp_{t,a} X_t X_t\transpose\theta_{t,a} \II{A_t=a}}^2} 
	\\
	&\le \EEt{\EEcc{\sum_{a = 1}^K  \trace{\pi_t(a|\tX)   \tX\tX\transpose \Sp_{t,a} X_t X_t\transpose \Sp_{t,a} 
				\II{A_t=a}}}{\tX}}
	\\
	&= \sum_{a = 1}^K   \EEt{ \trace{\Sigma_{t,a} \Sp_{t,a} \Sigma_{t,a} \Sp_{t,a}} },
	\end{align*}
	where we used $\biprod{\tX}{\theta_{t,a}} \le 1$ in the inequality and observed that $\Sigma_{t,a} = \EEt{\pi_t(a|X_0)X_0X_0\transpose}$.
	In what follows, we suppress the $t,a$ indexes to enhance readability. Using the Araki–Lieb–Thirring inequality, we get
	
	\[  \trace{\Sigma \Sigma^+ \Sigma \Sigma^+} \le \trace{\Sigma^2 \pa{\Sigma^+}^2 }.\]
	
	Define $C_k = \pa{I - \beta B_k}$. Using the definition of 
	$\Sp$ and elementary manipulations, we can get
	\begin{align*}
		\beta^{-2}\Spt& =\beta^{-2} \pa{\beta I + \beta \sum_{k=1}^M \prod_{j=1}^k C_i}^2 = I+2\sum_{k=1}^{M}\prod_{j=1}^k C_j + \sum_{k,k'=1}^{M} \pa{\prod_{j=1}^k C_j }\pa{\prod_{j=0}^{k'} C_j }\\
		& = I + 2\sum_{k=1}^{M}\prod_{j=1}^k C_j + 2 \sum_{k=1}^{M}\sum_{k'=k}^{M} \prod_{j=1}^k C^2_j \prod_{j=k+1}^{k'} C_j - \sum_{k=1}^{M} \prod_{j=1}^k C^2_j\\
		& \preccurlyeq 2I + 2 \sum_{k=1}^{M}\prod_{j=1}^k C_j  + 2 \sum_{k=1}^{M}\sum_{k'=k}^{M} \prod_{j=1}^k C^2_j \prod_{j=k+1}^{k'} C_j,
	\end{align*}
	where in the second line we reordered the sum $\sum^M_{k,k' = 1} a_k a_{k'} = 2 \sum_{k=1}^{M} \sum_{k'=k}^{M}a_k a_{k'} - \sum_{k=1}^{M} a^2_k$, while in the third line we dropped the last term and added $I$. Denote $D=\EEt{ C_j}$ and $E = \EEt{C^2_j}$. Using independence of $C_j$'s we get:
	\[ \beta^{-2}\EEt{\Spt} \preccurlyeq 2 \sum_{k=0}^{M} D^k +2 \sum_{k=1}^{M} E^k \sum_{k'=0}^{M-k}D^{k'}.\]
	Using the fact that $D=I-\beta \Sigma$, we have $\beta \sum_{k=0}^{M}D^k = \Sigma^{-1} - D^M \Sigma^{-1}$ and thus:
	\[ \beta^{-1} \EEt{\Spt } \preccurlyeq 2 \pa{I - D^M}\Sigma^{-1} + 2\sum_{k=1}^{M} E^k \pa{I-D^{M-k}}\Sigma^{-1}. \]
	We now use the fact that if $A\preccurlyeq B$, then for any positive semi-definite matrix $C$ holds the inequality $\trace{CA}\le \trace(CB)$ to get
	\begin{align*}
		\trace{\Sigma^2 \EEt{\Spt} } &= \trace{\EEt{\Spt} \Sigma^2 }\\
		&\le 2 \beta \trace{\Sigma - D^M\Sigma } + 2 \beta \sum_{k=1}^{M}\trace{ E^k \Sigma} - 2\beta \sum_{k=1}^{M}\trace{ E^k D^{M-k} \Sigma}.
	\end{align*} 
Since $D^m$ for any $m$ commutes with $\Sigma$ and $D^m\Sigma  \succcurlyeq 0 $, while $E^k$ is positive semi-definite, we can drop negative terms $\trace{D^M \Sigma}$ and $\trace{ E^k D^{M-k} \Sigma}$. Furthermore, as long as $\beta B  \preccurlyeq I$,
\[ E = \EEt{(I-\beta B)^2}  \preccurlyeq \EEt{(I-\beta B)} = D, \]
so that
\[ \trace{\Sigma^2 \EEt{\Spt} } \le  2 \beta \trace{\Sigma } + 2 \beta \sum_{k=1}^{M}\trace{ D^k \Sigma} =  2 \beta \trace{\Sigma } + 2 \trace{ \pa{\Sigma^{-1} - D^M \Sigma^{-1}} \Sigma} \le 3d.  \]

\jmlrQED

\section{Algorithms for contextual learning with full information}\label{sec:fullinfo}
Clearly, our algorithm \linexp can be simply adapted to simpler settings where the 
learner gets more feedback about the loss functions $\ell_t$ chosen by the adversary.
In this section, we show results for two such natural settings: one where the learner 
observes the \emph{entire} loss function $\ell_t$, and one where the learner observes
the losses $\ell_t(X_t,a)$ for each action $a$. We refer to the first of these observation 
models as \emph{counterfactual feedback} and call the second one \emph{full-information 
feedback}. We describe two variants of our algorithm for these settings and give their 
performance guarantees below. Both results will hold for general nonlinear losses taking values in  
$[0,1]$.

In case of counterfactual feedback, we can modify our algorithm so that, in each round $t$,
it computes the weights $w_{t,a}(X_t) = \exp\pa{-\eta\sum_{k=1}^{t-1} \ell_k(X_t,a)}$ for each  
action, and then plays action $A_t=a$ with probability proportional to the obtained weight. Using 
our general analytic tools, this algorithm can  be easily shown to achieve the following guarantee:
\begin{proposition}
For any $\eta>0$, the regret of the algorithm described above for counterfactual feedback satisfies
\[
 R_T \le \frac{\log K}{\eta} + \frac{\eta T}{8}.
\]
Setting $\eta = \sqrt{\frac{8\log K}{T}}$, the regret also satisfies $R_T \le \sqrt{(T/2)\log K }$.
\end{proposition}
Notably, this bound does not depend at all on the dimension of the context space, the complexity 
of the policy class, or any property of the loss function, and only shows dependence on the 
number of actions $K$. The caveat is of course that the 
counterfactual model provides the learner with a level of feedback that is entirely unrealistic 
in any practical setting: it requires the ability to evaluate all past loss functions at \emph{any} 
context-action pair.

The full-information setting is arguably much more realistic in that it only requires evaluating the 
losses corresponding to the observed context $X_t$, which which is typically the case in online 
classification problems. For this setting, we use our \linexp 
algorithm with the loss estimator defined for each action $a$ as
\[
 \hloss_{t,a} = \Sigma^{-1} X_t \loss_t(X_t,a).
\]
Using our analysis, we can show that the bias of this estimator is uniformly bounded by 
$\varepsilon\sqrt{d}$ (cf.~Equation~\ref{eq:robustbias}). The following bound is then easy to prove 
by following the same steps as in Section~\ref{sec:robust}:
\begin{proposition}
For any positive $\eta \le \frac{\lambdamin}{\sigma^2}$, the regret of the algorithm described above for 
full-information feedback
\[
 R_T \le \frac{\log K}{\eta} + \eta d T + \varepsilon\sqrt{d} T.
\]
Setting $\eta = \sqrt{\frac{d\log K}{T}}$, the regret also satisfies $R_T \le 2\sqrt{dT\log K} + 
\varepsilon\sqrt{d} T$ for large enough $T$.
\end{proposition}
As expected, this bound scales with the dimension as $\sqrt{d}$ due to the fact that the algorithm 
has to ``estimate'' $d$, parameters, as opposed to the $Kd$ parameters that need to be learned in 
the contextual bandit problem we consider in the main text. We also note that this online learning 
setting is closely related to that of prediction with expert advice, with the set of experts being 
the class of linear-classifier policies \citep{CBLu06:book}. As a result, it is possible to make 
use of this framework by running any online prediction algorithm on a finely discretized set of 
policies, resulting in a regret bound of order $\sqrt{dT\log(KT)}$. Our result above improves on 
this by a logarithmic factor of $T$, while being efficient to implement.

\section{Efficient implementation of MGR}
\label{appendixC}
The na\"ive implementation of the MGR procedure presented in the main text requires $O(MKd + Kd^2)$ 
time due to the matrix-matrix multiplications involved. In this section we explain how to compute 
$\hloss_t(x,a) =  \biprod{x}{\ttheta_{t,a}}$ in $O(MKd)$ time, exploiting the fact that the 
matrices $\wh{\Sigma}_{t,a}$ never actually need to be computed, since the algorithm only works 
with products of the form $\wh{\Sigma}_{t,a}X_t$ for a fixed vector $X_t$. This motivates 
the following procedure:

\vspace{.25cm}
\makebox[\textwidth][c]{
	\fbox{
		\begin{minipage}{.6\textwidth}
			\textbf{Fast Matrix Geometric Resampling}
			\vspace{.1cm}
			\hrule
			\vspace{.1cm}
			\textbf{Input:} context vector $x$, data distribution $\Dw$, policy $\pi_t$.\\
			\textbf{Initialization:} Compute $Y_{0,a} = I  x$.\\
			\textbf{For $k = 1, \dots, M$, repeat}:
			\begin{enumerate}
				\vspace{-2mm}
				\item Draw $X(k)\sim\Dw$ and $A(k) \sim \pi_t(\cdot|X(k))$,
				\vspace{-2mm}
				\item if $a = A(k)$, set\\ $Y_{k,a} = Y_{k-1,a} - 
\beta\iprod{Y_{k-1,a}}{X(k)}X(k),$ 
                \vspace{-2mm}   
				\item otherwise, set $Y_{k,a} =Y_{k-1,a}.$ 
			\end{enumerate}
			\vspace{-2mm}
			\textbf{Return $q_{t,a} = \beta Y_{0,a} + \beta \sum_{k=1}^M Y_{k,a}$.} 
		\end{minipage}
	}
}
\vspace{.25cm}

\noindent It is easy to see from the above procedure that each iteration $k$ can be computed using 
$(K+1)d$ vector-vector multiplications: sampling each action $A(k)$ takes $Kd$ time due to having to 
compute the products $\biprod{X(k)}{\hTheta_{t,a}}$ for each action $a$, and updating $Y_{k,a}$ can 
be done by computing the product $\iprod{Y_{k-1,a}}{X(k)}$. Overall, this results in a total runtime 
of order $MKd$ as promised above.

\end{document}